\documentclass[letterpaper]{article} 
\usepackage{aaai2026}  
\usepackage{times}  
\usepackage{helvet}  
\usepackage{courier}  
\usepackage[hyphens]{url}  
\usepackage{graphicx} 
\usepackage{natbib}  
\usepackage{caption} 
\usepackage{algorithm}
\usepackage{algorithmic}

\usepackage{newfloat}
\usepackage{listings}
\DeclareCaptionStyle{ruled}{labelfont=normalfont,labelsep=colon,strut=off} 
\floatstyle{ruled}
\newfloat{listing}{tb}{lst}{}
\floatname{listing}{Listing}

\usepackage{amssymb}
\usepackage{amsmath}
\usepackage{booktabs}
\usepackage{multirow}
\usepackage{amsthm}
\newtheorem{theorem}{Theorem}
\newtheorem{corollary}[theorem]{Corollary}
\title{GateRA: Token-Aware Modulation for Parameter-Efficient Fine-Tuning}
\author{
    Jie Ou\textsuperscript{\rm 1}, 
    Shuaihong Jiang\textsuperscript{\rm 1}, 
    Yingjun Du\textsuperscript{\rm 2}\thanks{Corresponding author.}, 
    Cees G. M. Snoek\textsuperscript{\rm 2}
}
\affiliations{
    \textsuperscript{\rm 1}Yuanzhigu Technology Co., Ltd., Chengdu, China\\ 

    \textsuperscript{\rm 2} University of Amsterdam, Amsterdam, Netherlands\\

    oujieww6@gmail.com, y.du@uva.nl
}

\usepackage{bibentry}

\begin{document}

\maketitle

\begin{abstract}
Parameter-efficient fine-tuning (PEFT) methods, such as LoRA, DoRA, and HiRA, enable lightweight adaptation of large pre-trained models via low-rank updates.  
However, existing PEFT approaches apply static, input-agnostic updates to all tokens, disregarding the varying importance and difficulty of different inputs. This uniform treatment can lead to overfitting on trivial content or under-adaptation on more informative regions, especially in autoregressive settings with distinct prefill and decoding dynamics.
In this paper, we propose \textbf{GateRA}, a unified framework that introduces token-aware modulation to dynamically adjust the strength of PEFT updates. By incorporating adaptive gating into standard PEFT branches, GateRA enables selective, token-level adaptation—preserving pre-trained knowledge for well-modeled inputs while focusing capacity on challenging cases. Empirical visualizations reveal phase-sensitive behaviors, where GateRA automatically
suppresses updates for redundant prefill tokens while emphasizing adaptation during decoding.
To promote confident and efficient modulation, we further introduce an entropy-based regularization that encourages near-binary gating decisions. This regularization prevents diffuse update patterns and leads to interpretable, sparse adaptation without hard thresholding.
Finally, we present a theoretical analysis showing that GateRA induces a soft gradient-masking effect over the PEFT path, enabling continuous and differentiable control over adaptation. Experiments on multiple commonsense reasoning benchmarks demonstrate that GateRA consistently outperforms or matches prior PEFT methods.
\end{abstract}
\section{Introduction}

Large pre-trained models have revolutionized natural language processing and vision-language tasks, yet their massive parameter counts make full fine-tuning increasingly impractical. Parameter-efficient fine-tuning (PEFT) methods, such as LoRA~\cite{hu2021lora}, DoRA~\cite{liu2024dora}, and HiRA~\cite{huang2025hira}, address this by injecting trainable low-rank updates into frozen backbones, enabling task adaptation with minimal parameter overhead.

Despite their success, existing PEFT methods such as LoRA, DoRA, and HiRA typically apply static low-rank updates with a fixed adaptation strength across all tokens, without considering differences in content, position, or confidence.
This uniform strategy assumes that all tokens benefit equally from fine-tuning. In reality, however, many tokens that are frequent or structurally simple are already well captured by the pre-trained backbone. Others, such as tokens that are domain-specific or contextually ambiguous, may require stronger adaptation.
Applying the same update to all tokens can lead to inefficient use of capacity and, in some cases, overfitting to irrelevant features.

This limitation is especially problematic in autoregressive generation, where it persists across different phases (prefill and decoding) with the model behaving differently in each phase, and even different structures within the same layer exhibiting distinct behavioral patterns. 
As shown in Figure~\ref{fig:alpha_phase}, our method learns to assign near-zero modulation weights to in-distribution tokens while amplifying adaptation for out-of-distribution tokens, where predictive uncertainty is higher and error propagation more severe. 
This behavior suggests that fine-tuning should be applied selectively and adaptively, motivating the need for a token-aware gating mechanism to modulate update strength at a finer granularity.

To overcome the limitations of existing PEFT methods that apply uniform adaptation across all tokens, we propose GateRA, a unified framework that introduces token-aware modulation into low-rank adaptation. \textit{First}, GateRA dynamically adjusts the strength of adaptation for each token based on input content, enabling more efficient allocation of adaptation capacity by focusing on ambiguous or task-critical tokens while preserving the pre-trained knowledge for others. \textit{Second}, we incorporate an entropy-based regularization to guide the gating function toward confident, near-binary decisions, which enhances interpretability and leads to sparse, selective adaptation patterns. \textit{Third}, we theoretically analyze how GateRA impacts gradient flow and show that its gating mechanism induces a soft masking effect that suppresses noisy updates while preserving informative gradients, thereby achieving a better trade-off between plasticity and stability.

\begin{figure*}[t]
    \centering
    \includegraphics[width=\textwidth]{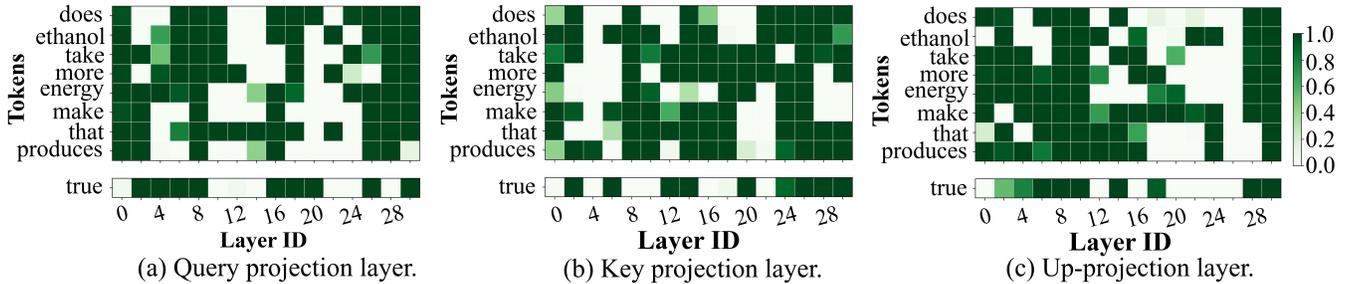}
    \caption{
    \textbf{Visualization of learned token-wise modulation values $g(x)$ across different layers of a decoder-only language model.}
    We observe that many tokens are assigned near-zero modulation weights, particularly in early self-attention layers, suggesting that the pre-trained weights are already sufficient for these inputs.
    This sparsity reveals that full adaptation is not necessary for all tokens, and highlights the need for selective, token-aware fine-tuning.
    }
    \label{fig:alpha_phase}
\end{figure*}

Empirically, we evaluate GateRA on a diverse set of commonsense reasoning and autoregressive generation benchmarks. The method consistently matches or surpasses existing PEFT approaches such as LoRA, DoRA, and HiRA, while introducing only a small number of additional parameters at inference time.
Furthermore, visualizations of token-level gating scores reveal clear, distribution-sensitive behavior: GateRA suppresses updates for in-distribution tokens that align with the model's existing knowledge and amplifies them when the model encounters out-of-distribution tokens with higher uncertainty. This behavior aligns well with human intuition and highlights the interpretability of our framework.

\section{Related Work}

\paragraph{Parameter-Efficient Fine-Tuning.}
Parameter-efficient fine-tuning (PEFT) methods aim to adapt large-scale pre-trained models to downstream tasks while minimizing the number of trainable parameters.
Early approaches include adapter tuning~\cite{houlsby2019parameter}, prefix-tuning~\cite{li2021prefix}, and prompt-tuning~\cite{lester2021power}, which insert task-specific modules or learn prompt embeddings. 
LoRA~\cite{hu2021lora} further improves efficiency by injecting trainable low-rank matrices into weight projections, achieving strong performance across a range of language and vision tasks.
While these methods reduce memory and compute costs, they typically apply uniform updates across all tokens, lacking the flexibility to differentiate token-level importance.
\textit{In contrast, our method introduces a token-aware gating mechanism that modulates the update strength dynamically per-token, allowing fine-grained control over adaptation.}

\paragraph{Extensions and Variants of LoRA.}
Numerous LoRA variants have been proposed to improve flexibility and robustness. 
AdaLoRA~\cite{zhang2023adaptive} adapts rank allocation across layers based on sensitivity; QLoRA~\cite{dettmers2023qlora} combines LoRA with quantization for low-resource fine-tuning.
DoRA~\cite{liu2024dora} decomposes weights into direction and magnitude for improved gradient flow, while MoRA~\cite{jiang2024mora}  employs a square matrix for high-rank adaptation. 
LoRASC~\cite{li2024lorasc} utilizes cascaded learning and slow-fast updates, and other efforts explore structure-aware improvements such as tensorized~\cite{yang-etal-2024-loretta}, Fourier-based~\cite{borse2024foura}, or expert-based~\cite{zhang2024milora} designs.
\textit{Our method is orthogonal and complementary to these extensions, focusing on dynamic token-wise modulation that can be applied on top of additive (LoRA), directional (DoRA), or multiplicative (HiRA) PEFT variants.}

\paragraph{Dynamic and Selective Adaptation.}
A few recent works have begun to explore input-dependent or sparse adaptation strategies.
Apart~\cite{qi2025adaptive} uses instance-wise adapter selection via a routing mechanism; UNIPELT~\cite{mao2022unipelt} learns a task-specific gating of multiple adaptation modules.
In vision, TR-PTS~\cite{luo2025tr} selectively activates PEFT modules based on token and task.
These methods introduce task-level or layer-level~\cite{yao2024layer} control, but rarely operate at the fine-grained token level during inference.
\textit{Our approach differs by directly modeling token-level update decisions through a gating function trained end-to-end, and further regularized to produce sparse, interpretable modulation across time.}

\section{Method}

\subsection{Overview of GateRA}

Parameter-efficient fine-tuning (PEFT) enables the adaptation of large pre-trained models by injecting small, trainable modules into frozen backbones. Among various PEFT strategies, HiRA~\cite{huang2025hira} has demonstrated strong performance by applying multiplicative low-rank updates to the weight matrix. Specifically, HiRA modifies the pre-trained weight $W_0$ as follows:
\[
W' = (AB + 1) \cdot W_0,
\]
where $A \in \mathbb{R}^{d \times r}$ and $B \in \mathbb{R}^{r \times d}$ are low-rank matrices and $\cdot$ denotes element-wise multiplication. This formulation allows HiRA to scale or suppress different dimensions of the backbone weights based on learned structure.

However, a key limitation remains: HiRA applies the same low-rank update to all input tokens, regardless of their difficulty or informativeness. This uniform treatment fails to account for token-level variability in uncertainty, importance, or phase (e.g., prefill vs. decoding). For instance, during autoregressive generation, the model may require stronger adaptation for novelty but require little to no updates for tokens representing existing knowledge already well-captured by the model. Static updates may therefore lead to overfitting on trivial tokens or under-adaptation to hard cases.
\begin{figure}[t]
    \centering
    \includegraphics[width=\linewidth]{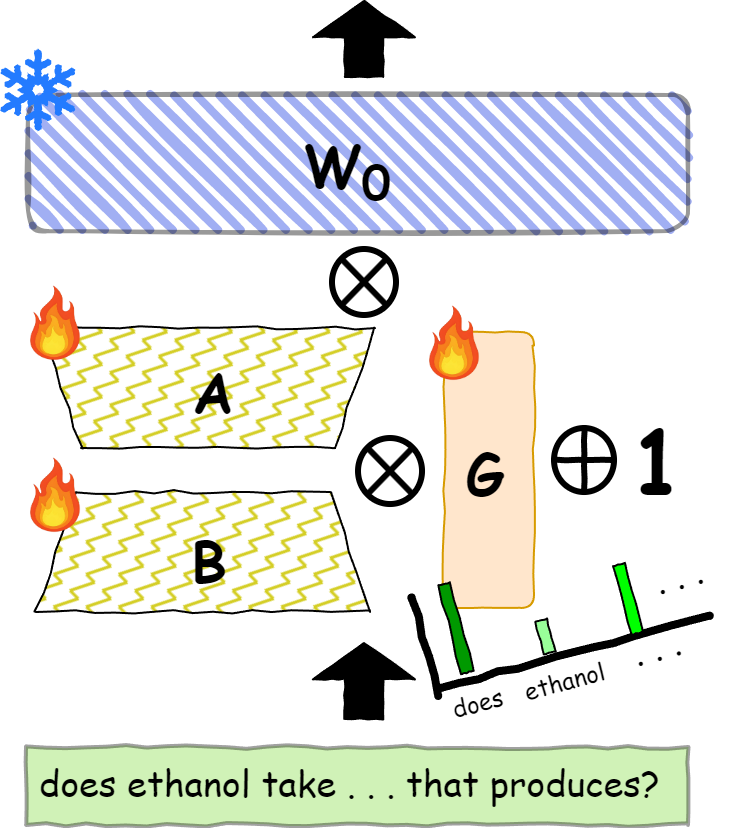}
\caption{
\textbf{Overview of the GateRA framework.} GateRA modulates the frozen pre-trained weights $W_0$ via a token-wise gating mechanism. A pair of low-rank matrices $A$ and $B$ generate parameter-efficient updates, which are scaled by a gating matrix $G$ and applied multiplicatively as $(G \odot (AB)+1) \cdot W_0$. This formulation enables dynamic, token-aware adaptation while preserving the backbone model's stability.
}
    \label{fig:framework}
\end{figure}

To address this, we propose \textbf{GateRA}, a simple yet effective extension to HiRA that introduces token-aware modulation. Instead of applying the same $(AB + 1)$ update for all tokens, GateRA replaces it with a dynamic form:
\[
W' = (\mathbf{g}(x) \cdot AB + 1) \cdot W_0,
\]
where $\mathbf{g}(x) \in [0, 1]$ is a gating scalar or vector predicted from the input token representation $x$. Intuitively, $\mathbf{g}(x)$ determines the strength of adaptation for each token, enabling selective fine-tuning at a finer granularity. Easy or redundant tokens may be suppressed (i.e., $\mathbf{g}(x) \approx 0$), while informative or uncertain ones receive amplified updates.

Figure~\ref{fig:framework} provides an overview of the GateRA architecture. A lightweight gating network is used to produce token-level modulation values, which control the flow of adaptation during both training and inference. This mechanism introduces only a small number of additional parameters and adds slight overhead during deployment and training. In the following sections, we describe the gating module design, our entropy-based regularization to promote confident gating, and a theoretical analysis showing how GateRA induces soft gradient masking for better knowledge preservation.

\subsection{Token-Aware Gating}

In standard PEFT methods such as HiRA, the low-rank adaptation term is applied uniformly across all tokens and time steps, regardless of their semantic importance or task-specific difficulty. However, as illustrated in Figure~\ref{fig:alpha_phase}, different tokens contribute unequally to model behavior, especially in autoregressive generation. For instance, tokens in the prefill stage often involve deterministic copying of input content and may not require adaptation, whereas tokens in the decoding stage typically involve uncertainty, reasoning, or generation from ambiguous contexts. Even within stages, module-level requirements differ. This motivates the need for a more flexible adaptation scheme that can selectively modulate parameter updates on a per-token basis.

\paragraph{Gating Module}

To address this, we introduce a lightweight gating module that dynamically controls the strength of low-rank adaptation based on the current token’s input representation. Concretely, given an input token embedding $x \in \mathbb{R}^d$, we compute a modulation coefficient via a small neural network:
\[
g(x) = \sigma(W_g x + b_g),
\]
where $W_g \in \mathbb{R}^{1 \times d}$ and $b_g \in \mathbb{R}$ are learnable parameters and $\sigma(\cdot)$ denotes the sigmoid activation. The scalar output $g(x) \in (0,1)$ adjusts the contribution of the low-rank component in a token-specific manner.

\paragraph{Modulated HiRA Update}

We instantiate this gating mechanism within the HiRA framework. Recall that HiRA adapts the frozen backbone via a multiplicative low-rank update of the form:
\[
W' = (AB + 1) \cdot W_0,
\]
where $AB$ is the low-rank term and $\cdot$ denotes element-wise multiplication. GateRA modifies this formulation by injecting the learned gating signal $g(x)$:
\[
W' = (g(x) \cdot AB + 1) \cdot W_0.
\]
In this way, the adaptation strength is no longer fixed but dynamically varies per input token. When $g(x) \approx 0$, the adaptation vanishes and the model relies entirely on the pre-trained weights. When $g(x) \approx 1$, full low-rank adaptation is applied.

This mechanism enables the model to learn context-aware adaptation strategies. 
Gating values in Figure~\ref{fig:alpha_phase} reflect semantic patterns: low for existing knowledge, high for novel tokens.
Such behavior is learned automatically from supervision signals, without any explicit phase annotation or token-level labels.
This token-aware control scheme allows GateRA to preserve pre-trained knowledge for trivial inputs while allocating adaptation capacity more judiciously, which we show leads to better generalization and interpretability.

\subsection{Entropy-Based Regularization}

While the gating mechanism introduced in GateRA enables token-level modulation of adaptation strength, unconstrained learning of $g(x)$ may lead to overly smooth or indecisive gating behaviors. In practice, we observe that without additional regularization, the model tends to assign ambiguous gating values (e.g., close to 0.5) to some tokens, resulting in uniformly soft updates across tokens and diminishing the benefits of selective adaptation.

To address this, we introduce an entropy-based regularization term that encourages confident and sparse gating decisions. Specifically, we treat each gating value $g(x) \in (0,1)$ as a Bernoulli probability and penalize its entropy:
\begin{align}
\mathcal{L}_{\text{ent}} 
&= \frac{1}{N} \sum_{i=1}^N H(g(x_i)) \notag \\
&= - \frac{1}{N} \sum_{i=1}^N \bigg[
    g(x_i) \log g(x_i) \notag \\
&\quad + (1 - g(x_i)) \log (1 - g(x_i))
\bigg]
\end{align}
where $N$ is the number of tokens in the batch. This term reaches its minimum when $g(x)$ approaches 0 or 1, and its maximum when $g(x) = 0.5$, thereby promoting near-binary gating.

The benefits of this regularization are twofold. First, it enhances interpretability by forcing the model to make discrete-like adaptation decisions, highlighting which tokens trigger updates. Second, it improves generalization by reducing the model’s tendency to over-adapt on trivial or noisy inputs, as gating values closer to 0 effectively suppress adaptation for low-importance tokens.

Empirically, we find that incorporating this entropy penalty yields both improved performance and crisper gating visualizations (cf. Figure~\ref{fig:alpha_phase}). It also facilitates downstream analysis by producing sparse token-level attribution maps, which can be used to better understand when and where adaptation is truly needed.

\section{Theoretical Analysis}

We now present a theoretical understanding of GateRA, highlighting how its token-aware gating mechanism modulates gradient flow and induces selective adaptation. Our analysis draws from the view that multiplicative PEFT can be interpreted as a dynamic reweighting of parameter updates. In contrast to prior methods such as HiRA, which apply a uniform scaling factor across all tokens, GateRA introduces a data-dependent modulation that adapts to token-level variation in uncertainty and informativeness.

\paragraph{Setup}

Let $x \in \mathbb{R}^d$ be the input token embedding and $W_0 \in \mathbb{R}^{d_\text{out} \times d}$ be the frozen pre-trained weight matrix. A PEFT method such as HiRA ($\gamma=1$) produces an adapted output of the form:
\[
y = (1 + \gamma \cdot AB) \cdot W_0 x,
\]
where $\gamma$ can be a learnable scalar (or layer-wise parameter). In GateRA, we instead use a token-specific modulation:
\[
y = (1 + g(x)\cdot AB) \cdot W_0 x,
\]
where $g(x) \in [0, 1]$ is a gating function implemented via a neural network followed by sigmoid activation. We denote the PEFT component as $W_\Delta = g(x)\cdot AB \cdot W_0$. 

Let $\mathcal{L}(y, t)$ denote a token-level loss with target $t$. Our analysis focuses on the gradients with respect to the PEFT component $W_\Delta$, and how gating modulates its contribution to training dynamics.

\subsection{Gradient Modulation Bound}

We first characterize how the gradient norm is bounded by the gating function, offering selective update control.

\begin{theorem}[Token-Aware Gradient Modulation]
\label{thm:grad-bound}
Let $\mathcal{L}(y, t)$ be convex and differentiable in $y$. Then, under the GateRA formulation with $W_\Delta = g(x)\cdot AB\cdot W_0$, the gradient norm with respect to the base adapter $AB$ satisfies:
\[
\left\|\frac{\partial \mathcal{L}}{\partial AB}\right\|_F \leq g(x) \cdot \|W_0\| \cdot \left\|\frac{\partial \mathcal{L}}{\partial y}\right\| \cdot \|x\|.
\]
\end{theorem}

\begin{proof}
We have:
\[
y = W_\Delta x = g(x)\cdot AB\cdot W_0 x.
\]
By the chain rule, the gradient of the loss with respect to $AB$ is:
\[
\frac{\partial \mathcal{L}}{\partial AB} = \frac{\partial \mathcal{L}}{\partial y} \cdot \frac{\partial y}{\partial AB} = g(x) \cdot W_0 \cdot \frac{\partial \mathcal{L}}{\partial y} \cdot x^\top.
\]
Taking the Frobenius norm:
\[
\left\| \frac{\partial \mathcal{L}}{\partial AB} \right\|_F \leq g(x) \cdot \|W_0\| \cdot \left\| \frac{\partial \mathcal{L}}{\partial y} \right\| \cdot \|x\|,
\]
as required.
\end{proof}

\noindent This result shows that the magnitude of updates to the adaptation branch is directly controlled by $g(x)$: tokens with small gating values will receive negligible updates, preserving pre-trained features, while tokens with large $g(x)$ enable strong task-specific adaptation.

\subsection{Selective Adaptation and Regularization}

Theorem~\ref{thm:grad-bound} reveals that GateRA effectively implements a \textit{soft masking mechanism} over gradients. This leads to the following corollary:

\begin{corollary}[Selective Gradient Suppression]
For tokens where $g(x) \to 0$, the PEFT gradient vanishes:
\[
\lim_{g(x) \to 0} \left\|\frac{\partial \mathcal{L}}{\partial W_\Delta}\right\| = 0.
\]
Thus, GateRA preserves pre-trained knowledge for confidently modeled tokens and avoids overfitting.
\end{corollary}

\noindent In practice, this mechanism allows the model to focus adaptation capacity on harder tokens while skipping trivial or well-represented ones.

Moreover, the entropy regularization on $g(x)$ plays a crucial role in promoting confident decisions:
\[
\mathcal{L}_{\text{ent}} = \mathbb{E}_{x} \left[ -g(x) \log g(x) - (1 - g(x)) \log (1 - g(x)) \right].
\]

This encourages $g(x)$ to approach binary values (0 or 1), reinforcing the soft masking effect while maintaining differentiability. It naturally promotes sparse and interpretable adaptation patterns without requiring hand-crafted thresholds.

\begin{table*}[t]
    \centering
    \resizebox{\textwidth}{!}{
\begin{tabular}{llcccccccccc}
\toprule
Model & Method & Params(\%) & BoolQ & PIQA & SIQA & ARC-c & ARC-e & OBQA & HelaS & WinoG & Average \\
\midrule
ChatGPT & - & - & 73.10 & 85.40 & 68.50 & 79.90 & 89.80 & 74.80 & 78.50 & 66.10 & 77.01 \\
\midrule
\multirow{7}{*}{L2-7B} & Prompt Tuning & 0.0012 & 55.93 & 12.35 & 30.50 & 6.06 & 8.63 & 9.40 & 6.91 & 40.57 & 21.29 \\
 & P-Tuning & 0.7428 & 58.75 & 36.02 & 0.20 & 0.17 & 1.98 & 0.80 & 0.01 & 0.00 & 12.24 \\
 & LoRA  & 0.8256 & 69.80 & 79.90 & 79.50 & 64.70 & 79.80 & 81.00 & 83.60 & 82.60 & 77.61 \\
 & DoRA  & 0.8256 & 71.80 & 83.70 & 76.00 & 68.20 & 83.70 & 82.40 & 89.10 & 82.60 & 79.69 \\
 & MoRA  & 0.8241 & 72.17 & 80.79 & 79.53 & 71.42 & 85.31 & 81.20 & 29.09 & 80.19 & 72.46 \\
 & HiRA  & 0.8256 & 71.22 & 83.35 & 79.53 & 73.81 & 86.74 & 84.60 & 88.12 & 83.98 & 81.42 \\
 &  \textbf{GateRA} &   0.8384&  \textbf{72.84}& \textbf{84.39}& \textbf{80.66}& \textbf{75.26}& \textbf{88.22}& \textbf{85.40}& \textbf{88.58}& \textbf{84.85}& \textbf{82.52}\\
\midrule
\multirow{7}{*}{L3-8B} & Prompt Tuning & 0.0010 & 56.85 & 45.05 & 36.13 & 31.57 & 32.74 & 29.20 & 14.01 & 50.12 & 36.96 \\
 & P-Tuning & 0.6240 & 59.97 & 11.64 & 8.19 & 7.42 & 8.63 & 9.60 & 1.77 & 37.65 & 18.11 \\
 & LoRA  & 0.7002 & 70.80 & 85.20 & 79.90 & 71.20 & 84.20 & 79.00 & 91.70 & 84.30 & 80.79 \\
 & DoRA  & 0.7002 & 74.60 & 89.30 & 79.90 & 80.40 & 90.50 & 85.80 & 95.50 & 85.60 & 85.20 \\
 & MoRA  & 0.6997 & 74.28 & 87.43 & 80.71 & 79.61 & 91.16 & 85.60 & 43.53 & 86.74 & 78.63 \\
 & HiRA  & 0.7002 & 75.40 & \textbf{89.70} & 81.15 & 82.90 & 93.27 & \textbf{88.32} & 95.36 & 87.70 & 86.72 \\
 & \textbf{GateRA} &  0.7123 &  \textbf{75.72}& 89.45& \textbf{82.19}& \textbf{85.15}& \textbf{93.64}& 87.60& \textbf{96.21}& \textbf{90.29}& \textbf{87.53}\\
\bottomrule
\end{tabular}
}
    \caption{\textbf{Accuracy comparison on commonsense reasoning tasks.}
We report results on eight subtasks from the CommonsenseQA benchmark.
GateRA consistently outperforms all baselines under both LLaMA-2-7B (L2-7B) and LLaMA-3-8B (L3-8B). 
}
    \label{tab:commonsense}
\end{table*}
\subsection{Comparison to Prior Work}

Unlike HiRA, where the modulation scalar $\gamma$ is static and shared across all tokens, GateRA enables fine-grained, token-wise gradient control. This allows:
Token-level interpretability: which tokens are updated can be visualized via $g(x)$ (cf. Figure~\ref{fig:alpha_phase}); Dynamic generalization: GateRA prevents unnecessary updates on trivial tokens, improving out-of-distribution robustness; Better optimization stability: the soft, differentiable gating avoids the brittleness of hard masking. We thus provide a theoretical foundation for the observed empirical benefits of GateRA, and establish a rigorous connection between its gating mechanism and selective representation learning.

\section{Experiment}

We evaluate the proposed \textbf{GateRA} method on three representative tasks to assess its effectiveness in reasoning-intensive settings: commonsense reasoning, open-domain dialogue generation, and mathematical reasoning. These tasks span diverse formats (classification, generation, symbolic reasoning), making them suitable benchmarks to assess the generalization and robustness of token-aware PEFT.

\subsection{Datasets}

\paragraph{Commonsense Reasoning.}
We follow~\cite{hu2023llm} and adopt eight widely used sub-tasks with predefined training and test splits, totaling 170,420 query-answer pairs. The sub-tasks include: \textbf{BoolQ}~\cite{clark2019boolq}: yes/no questions;
\textbf{PIQA}~\cite{bisk2020piqa}: physical commonsense;
\textbf{SIQA}~\cite{sap2019socialiqa}: social reasoning;
\textbf{HellaSwag}~\cite{zellers2019hellaswag}: sentence completion;
\textbf{WinoGrande}~\cite{sakaguchi2021winogrande}: coreference-based commonsense;
\textbf{ARC-e} and \textbf{ARC-c}\cite{clark2018think}: multiple-choice science QA;
\textbf{OBQA}~\cite{mihaylov2018can}: multi-step reasoning questions.
These datasets collectively cover binary, multiple-choice, and cloze-style questions, posing different levels of difficulty for LLMs.

\paragraph{Open-domain Dialogue Generation.}
We use the \textbf{ConvAI2} dataset~\cite{dinan2019second}, which consists of 17,878 multi-turn dialogues for training and 1,000 for testing. Each instance includes persona profiles and a conversational history. Following~\cite{liu2020you,song2021bob,huang2023personalized}, we choose a self-persona configuration where only the speaker’s persona is visible.

\paragraph{Mathematical Reasoning.}
We evaluate symbolic and multi-step reasoning capabilities on \textbf{GSM8K}~\cite{cobbe2021training}, a challenging benchmark consisting of grade-school math word problems requiring structured solution steps.

\subsection{Experimental Settings}

\paragraph{Evaluation Metric.}
For commonsense reasoning, we follow the keyword matching protocol in~\cite{liu2024dora, huang2025hira}, using accuracy as the primary metric. The decoded answer is scanned for specific keywords (e.g., “true”, “B”), and the first match is taken as the prediction. If no match is found, the answer is considered incorrect. This rule-based scheme enables consistent evaluation across different question formats.
For ConvAI2, we report BLEU \cite{papineni2002bleu} and BERTScore \cite{zhang2019bertscore} to measure generation quality and semantic similarity.
For GSM8K, we report accuracy by comparing the model’s final numeric prediction with the ground truth using an exact match.

\paragraph{Baselines.}
We compare \textbf{GateRA} against strong PEFT methods including LoRA~\cite{hu2021lora}, DoRA~\cite{liu2024dora}, MoRA~\cite{jiang2024mora}, and HiRA~\cite{huang2025hira}. To ensure fair comparison, we use the same injection locations (query/key/value/MLP) as HiRA. 

\begin{table*}[t]
\centering
\small
\setlength{\tabcolsep}{4pt}
\begin{tabular}{llcccccccc}
\toprule
Model & Method & Params (\%) & BLEU & BERT F1 & BERT-R & BERT-P & Meteor & R-L & Average \\
\midrule
\multirow{6}{*}{L2-7B}
& Prompt Tuning         & 0.0012 & 0.04 & 72.44 & 77.38 & 68.23 &  0.80 &  0.80 & 36.62 \\
& P-Tuning              & 0.7428 & 0.60 & 83.29 & 83.33 & 83.28 & 15.11 & 12.36 & 46.33 \\
& MoRA         & 0.8241 & 1.09 & 84.09 & 84.65 & 83.59 & 10.97 &  9.57 & 45.66 \\
& LoRA         & 0.8256 & 1.82 & 84.41 & 84.71 & 84.16 & 11.38 & 10.55 & 46.17 \\
& DoRA         & 0.8256 & 1.73 & 84.18 & 84.61 & 83.81 & 11.25 & 10.41 & 46.00 \\
& HiRA         & 0.8256 & {2.70} & {84.86} & {84.98} & {84.77} & 13.56 & {12.80} & 47.28 \\
& \textbf{GateRA} & 0.8384 & \textbf{2.75} & \textbf{85.63} & \textbf{85.73} & \textbf{85.54} & \textbf{14.37} & \textbf{13.74} & \textbf{47.96} \\
\midrule
\multirow{6}{*}{L3-8B}
& Prompt Tuning         & 0.0010 & 1.45 & 82.99 & 82.99 & 83.05 & 14.72 & 13.13 & 46.39 \\
& P-Tuning              & 0.6240 & 1.50 & 81.52 & 81.07 & 82.01 & 15.49 & 13.55 & 45.86 \\
& LoRA         & 0.7002 & 2.26 & 84.32 & 84.00 & 84.67 & 12.51 & 11.77 & 46.59 \\
& DoRA         & 0.7002 & 2.29 & 84.32 & 84.06 & 84.62 & 12.63 & 11.78 & 46.62 \\
& MoRA         & 0.6997 & 1.60 & 84.22 & 84.06 & 84.43 & 12.37 & 11.19 & 46.31 \\
& HiRA         & 0.7002 & {3.41} & {84.81} & {84.40} & {85.25} & 14.87 & {14.05} & 47.80 \\
& \textbf{GateRA} & 0.7123 & \textbf{3.53} & \textbf{85.74} & \textbf{85.31} & \textbf{86.16} & \textbf{15.79} & \textbf{15.13} & \textbf{48.61} \\
\bottomrule
\end{tabular}
\caption{
\textbf{Results on the CONVAI2 dialogue generation task.}
We evaluate BLEU, BERT-based F1/Recall/Precision, Meteor, and ROUGE-L.
}
\label{tab:convai2}
\end{table*}

\paragraph{Implementation Details.}
Following~\cite{liu2024dora}, we evaluate on LLaMA-2-7B~\cite{touvron2023llama} and LLaMA-3-8B~\cite{grattafiori2024llama}. All models are fine-tuned using the AdamW optimizer~\cite{loshchilov2017decoupled} with a learning rate of 1e-3 and warmed up with 100 steps.

\subsection{Main Results}
\paragraph{Results on Commonsense Reasoning Tasks}

As shown in Table~\ref{tab:commonsense}, GateRA achieves superior performance across all commonsense reasoning benchmarks on both the LLaMA-2-7B and LLaMA-3-8B models. Specifically, GateRA obtains an average accuracy of \textbf{82.52\%} on LLaMA-2-7B, surpassing the best baseline HiRA (81.42\%) and outperforming LoRA-based methods such as DoRA and MoRA. On LLaMA-3-8B, GateRA achieves \textbf{87.53\%}, exceeding the previous best HiRA score of 86.72\%. GateRA also consistently outperforms baselines on challenging tasks like PIQA, OBQA, and WinoGrande, demonstrating the effectiveness of its token-aware adaptation. These results validate the design of GateRA, where selectively modulating adaptation strength per token improves generalization and efficiency. With a comparable parameter budget to LoRA variants, GateRA enables more adaptive and effective tuning.

\paragraph{Results on Conversational Task}

Table~\ref{tab:convai2} presents the results on the CONVAI2 dialogue generation benchmark. GateRA achieves the best performance on both LLaMA-2-7B and LLaMA-3-8B, with average scores of \textbf{47.96} and \textbf{48.61}, respectively. These results surpass HiRA and all other PEFT baselines across all evaluated metrics, including BLEU, BERT-based F1/Recall/Precision, Meteor, and ROUGE-L.  This further confirms the strength of our token-aware gating strategy in controlling adaptation behavior for open-domain generation.

\paragraph{Results on Mathematical Reasoning Tasks}

We evaluate the performance of GateRA on mathematical reasoning tasks using the MetaMath dataset ~\cite{yu2023metamath} for training and GSM8K for evaluation. As reported in Table~\ref{tab:math}, GateRA achieves state-of-the-art performance compared to all PEFT baselines under the LLaMA-3-8B setting. Specifically, GateRA attains an accuracy of \textbf{72.11\%}, outperforming HiRA (70.81\%), MoRA (67.98\%), DoRA (66.12\%), and LoRA (65.89\%). These findings demonstrate that GateRA is effective for complex 
reasoning, by leveraging token-aware modulation to enhance expressiveness while maintaining parameter efficiency.

\begin{table}[t]
\centering
\begin{tabular}{l|l|c|c}
\toprule
\textbf{Model} & {Method} & \textbf{Trainable} & \textbf{GSM8K} \\
\midrule
\multirow{7}{*}{L3-8B}
& Prompt Tuning & 0.0012 & 15.62 \\
& P-Tuning      & 0.7428 & 2.65  \\
& LoRA  & 0.7002 & 65.89 \\
& DoRA  & 0.7002 & 66.12 \\
& MoRA  & 0.6997 & 67.98 \\
& HiRA  & 0.7002 & 70.81 \\
& \textbf{GateRA}  & 0.7123 & \textbf{72.11} \\
\bottomrule
\end{tabular}
\caption{
\textbf{Results on mathematical reasoning tasks.}
We evaluate models on the GSM8K benchmark with MetaMath~\cite{yu2023metamath} training.
}
\label{tab:math}
\end{table}

\begin{table*}[t]
    \centering
\begin{tabular}{lccccccccc}
\toprule
\textbf{Component} & \textbf{BoolQ} & \textbf{PIQA} & \textbf{SIQA} & \textbf{ARC-c} & \textbf{ARC-e} & \textbf{OBQA} & \textbf{HellaS} & \textbf{WinoG} & \textbf{Average} \\
\midrule
FC, QKV & \textbf{75.72} & \textbf{89.45} & \textbf{82.19} & \textbf{85.15} & \textbf{93.64} & \textbf{87.60} & \textbf{96.21} & \textbf{90.29} & \textbf{87.53} \\
FC      & 74.28 & 89.12 & 81.12 & 82.68 & 93.27 & 87.00 & 96.00 & 88.79 & 86.53 \\
QV      & 73.80 & 88.95 & 80.84 & 81.95 & 93.10 & 86.80 & 95.63 & 88.35 & 86.18 \\
QKV     & 75.17 & 89.17 & 80.71 & 82.25 & 93.48 & 87.20 & 95.46 & 88.56 & 86.50 \\
QK      & 73.42 & 88.35 & 80.02 & 81.56 & 92.85 & 86.45 & 95.01 & 87.74 & 85.68 \\
V       & 70.58 & 87.94 & 78.92 & 79.70 & 91.89 & 85.71 & 94.00 & 85.88 & 84.33 \\
Q       & 71.36 & 86.92 & 78.42 & 79.23 & 90.90 & 85.35 & 93.61 & 85.62 & 83.93 \\
K       & 71.00 & 86.85 & 78.63 & 79.88 & 90.78 & 85.10 & 93.72 & 85.03 & 83.88 \\
\bottomrule
\end{tabular}
\caption{Performance of the LLaMA-3-8B model with \textbf{GateRA} integrated into various components. GateRA consistently achieves the best performance when applied jointly to the FC and QKV layers.}
\label{tab:component_ablation}
\end{table*}

\subsection{Ablation Studies}

\paragraph{Effect of Rank on Adaptation Capacity}

To evaluate the parameter-efficiency of GateRA, we compare its performance under two adaptation ranks: $r=16$ and $r=32$, across eight commonsense reasoning benchmarks. As illustrated in Figure~\ref{fig:rank_ablation}, GateRA with $r=16$ achieves an average accuracy of 86.70\%, which is highly competitive with the 87.53\% obtained by the $r=32$ setting. Remarkably, the lower-rank configuration requires only half the trainable parameters, yet maintains strong task-wise performance across all datasets, with minimal degradation. For instance, GateRA with $r=16$ achieves comparable or even superior performance to the higher-rank baseline on PIQA, ARC-E, and BoolQ. These results highlight the robustness of our token-aware gating mechanism, which effectively prioritizes informative tokens under limited adaptation budgets.

\begin{figure}[h]
    \centering
    \includegraphics[width=\linewidth]{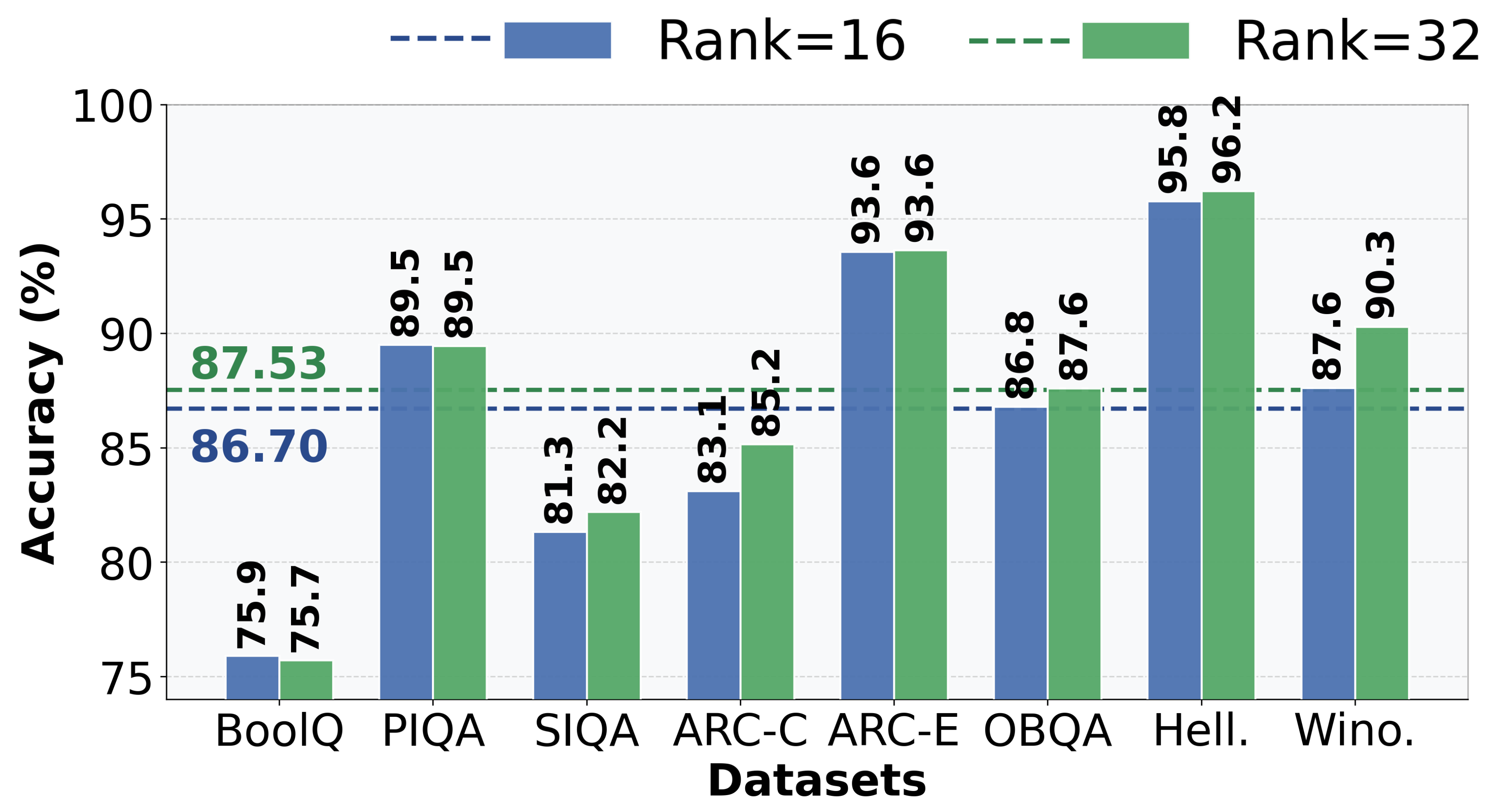}
    \caption{
    Token-adaptive performance under different ranks.
    }
    \label{fig:rank_ablation}
\end{figure}
\paragraph{Impact of Component-Wise Integration} 
We investigate how integrating GateRA into different subcomponents of the transformer affects performance across commonsense reasoning tasks. As shown in Table~\ref{tab:component_ablation}, GateRA yields the highest average accuracy (87.53\%) when applied jointly to both the final feedforward layer (FC) and the multi-head attention pathways (Q, K, V). This configuration consistently outperforms partial or isolated integration strategies. Among single-component variants, integrating into QV or QKV achieves better results than Q-only, V-only, or K-only setups, indicating that information flow through query and value channels is more crucial for downstream adaptation. Notably, integrating into FC alone performs competitively (86.53\%), highlighting the importance of adapting representations at the output level. Overall, these results confirm that GateRA benefits from a broader integration scope, and can still be effective when applied selectively.

\begin{table}[t]
\centering
\label{ab:para}
\begin{tabular}{l|l|c|c}
\toprule
\textbf{Model} & {Method} & \textbf{Trainable} & \textbf{Avg.} \\
\midrule
\multirow{4}{*}{L3-8B}
& w/o GateRA  & 0.7002 & 86.72 \\
& w/ Static-GateRA  & 0.7123 & 86.97 \\
& w/o Regularization  & 0.7123 & 87.08 \\
&  \textbf{GateRA}  & 0.7123 &  \textbf{87.53} \\
\bottomrule
\end{tabular}
\caption{
\textbf{Ablation study on gating mechanisms for commonsense reasoning tasks.} 
We compare GateRA with multiple variants to assess the effectiveness of its data-dependent modulation and regularization. Specifically, we analyze: (1) a HiRA baseline without any gating (w/o GateRA), (2) a static gating variant that replaces $g(x)$ with a learnable tensor of the same shape (Static-GateRA), and (3) GateRA without entropy regularization (w/o Regularization). 
}
\end{table}

\paragraph{Effectiveness of Data-Driven Gating.}
To assess the benefit of our proposed token-aware modulation function $g(x)$, we compare \textbf{GateRA} with a variant that replaces $g(x)$ with a learnable but static gating tensor of the same shape (\textbf{Static-GateRA}). While this variant maintains the same parameter count, it lacks input-awareness and results in lower performance (86.97\% vs. 87.53\%). This highlights the importance of data-driven gating, which allows the model to adaptively modulate updates based on input content, selectively attending to challenging or ambiguous tokens during reasoning.

\paragraph{Role of Entropy-Based Regularization.}
We further investigate the impact of the entropy-based regularization term that encourages confident and sparse gating decisions. Removing this term (w/o Regularization) leads to a slight performance degradation (87.08\% vs. 87.53\%), suggesting that entropy regularization plays a key role in stabilizing gate behaviors. It helps GateRA avoid overly diffused updates and promotes near-binary gating, which improves interpretability and generalization.

\section{Conclusion}

We propose \textbf{GateRA}, a unified and lightweight framework for test-time adaptation in PEFT methods. GateRA introduces a token-aware modulation mechanism that dynamically adjusts the strength of magnitude of weight usage based on epistemic uncertainty, enabling selective and input-specific adaptation. This allows the model to concentrate updates on uncertain or task-critical tokens while preserving generalizable representations. We further connect the gating function to a spike-and-slab prior and introduce entropy-based regularization to promote sparse and confident adaptation. Theoretically, GateRA induces a soft masking effect on gradient flow, improving the balance between stability and adaptability. Extensive experiments across multiple NLP benchmarks demonstrate that GateRA consistently improves over PEFT baselines. 

\section{Acknowledgments}
This work was supported in part by the University of Amsterdam and the Informatics Institute. Cees G. M. Snoek is (partially) funded by the Horizon Europe project ELLIOT (GA No. 101214398).

\bibliography{ref}

\end{document}